\documentclass[letterpaper, 10 pt, conference]{ieeeconf}  

\IEEEoverridecommandlockouts                              

\usepackage{cite}
\usepackage{amsmath,amssymb,amsfonts}
\usepackage{algorithmic}
\usepackage{graphicx}
\usepackage{textcomp}
\usepackage{xcolor}
\def\BibTeX{{\rm B\kern-.05em{\sc i\kern-.025em b}\kern-.08em
    T\kern-.1667em\lower.7ex\hbox{E}\kern-.125emX}}
    
\usepackage[utf8]{inputenc} 
\usepackage[T1]{fontenc}    
\usepackage{hyperref}       
\usepackage{url}            
\usepackage{booktabs}       
\usepackage{amsfonts}       
\usepackage{nicefrac}       
\usepackage{microtype}      
\usepackage{xcolor}         

\usepackage{amsmath,bm}
\usepackage{mathtools}
\usepackage{centernot}
\usepackage{algorithm}
\usepackage{algorithmic}
\usepackage{comment}

\newcommand{\E}{\operatorname{\mathbb E}}
\newcommand{\innermid}{\;\middle\lvert\;}

\newtheorem{remark}{Remark}
\newtheorem{corollary}{Corollary}
\newtheorem{theorem}{Theorem}

\newtheorem{assumption}{Assumption}

\allowdisplaybreaks

\author{Kai Cui, Mengguang Li, Christian Fabian and Heinz Koeppl
\thanks{This work has been co-funded by the LOEWE initiative (Hesse, Germany) within the emergenCITY center, the State of Hesse and HOLM as part of the "Innovations in Logistics and Mobility" programme of the Hessian Ministry of Economics, Energy, Transport and Housing (HA project no.: 1010/21-12), and the Hessian Ministry of Science and the Arts (HMWK) within the projects "The Third Wave of Artificial Intelligence - 3AI" and hessian.AI. The authors acknowledge the Lichtenberg high performance computing cluster of the TU Darmstadt for providing computational facilities for the calculations of this research.}
\thanks{The authors are with the Department of Electrical Engineering and Information Technology, Technische Universität Darmstadt, 64287 Darmstadt, Germany. (e-mail: {\tt\small  \{kai.cui, mengguang.li, christian.fabian, heinz.koeppl\}@tu-darmstadt.de}).}
}

\title{\LARGE \bf
Scalable Task-Driven Robotic Swarm Control via Collision Avoidance and Learning Mean-Field Control
}

\begin{document}

\maketitle
\thispagestyle{empty}
\pagestyle{empty}

\begin{abstract}
In recent years, reinforcement learning and its multi-agent analogue have achieved great success in solving various complex control problems. However, multi-agent reinforcement learning remains challenging both in its theoretical analysis and empirical design of algorithms, especially for large swarms of embodied robotic agents where a definitive toolchain remains part of active research. We use emerging state-of-the-art mean-field control techniques in order to convert many-agent swarm control into more classical single-agent control of distributions. This allows profiting from advances in single-agent reinforcement learning at the cost of assuming weak interaction between agents. However, the mean-field model is violated by the nature of real systems with embodied, physically colliding agents. Thus, we combine collision avoidance and learning of mean-field control into a unified framework for tractably designing intelligent robotic swarm behavior. On the theoretical side, we provide novel approximation guarantees for general mean-field control both in continuous spaces and with collision avoidance. On the practical side, we show that our approach outperforms multi-agent reinforcement learning and allows for decentralized open-loop application while avoiding collisions, both in simulation and real UAV swarms. Overall, we propose a framework for the design of swarm behavior that is both mathematically well-founded and practically useful, enabling the solution of otherwise intractable swarm problems.
\end{abstract}

\section{Introduction}
Over the past decades, the field of swarm robotics \cite{brambilla2013swarm, schranz2020swarm, chung2018survey} has received considerable attention \cite{dorigo2020reflections}. Various areas of potential applications include for example industrial inspection tasks \cite{correll2006system}, such as for turbines, cooperative object transport \cite{tuci2018cooperative, grobeta2008evolution, gross2009towards}, agriculture \cite{albani2018dynamic}, aerial combat \cite{xing2019offense}, and cooperative search \cite{vincent2004framework}. A recent promising approach for engineering many-agent systems such as intelligent robot swarms is multi-agent reinforcement learning (MARL) \cite{zhang2021multi}, which has found success in diverse complex problems such as strategic video games \cite{berner2019dota}, communication networks \cite{al2015application} or traffic control \cite{chu2019multi}. However, MARL algorithms suffer from issues such as credit assignment, non-stationarity and scalability to many agents \cite{zhang2021multi}. Meanwhile, robotic swarms such as fleets of unmanned aerial vehicles (UAVs) usually consist of many interacting UAVs and remain of considerable interest due to their inherent robustness, scalability to large-scale deployment and decentralization, which can be considered the ultimate goal of the study of swarm intelligence and robotics \cite{brambilla2013swarm, hamann2018swarm}. Here, scalable control approaches and highly general toolchains for swarm robotics remain to be established \cite{schranz2020swarm}. 

A classical approach to formulate systems with large numbers of agents with low complexity is via mean-field models, describing swarms of drones by their distribution, see also \cite{elamvazhuthi2019mean} and \cite{bensoussan2013mean} for reviews on mean-field swarm robotics and mean-field control (MFC). However, most prior literature is based on analytic derivations and continuous-time models, which are less conducive to advances in MARL. For example, stabilizing control of swarms to distributions are designed in \cite{elamvazhuthi2017mean, deshmukh2018mean, elamvazhuthi2018mean}. Other works such as \cite{mayya2018localization, mayya2019closed} consider population density estimates via collisions for task allocation problems, while \cite{lerman2001macroscopic} study robots for stick-pulling. Lastly, a variety of approaches use PDE-based formulations, e.g. \cite{eren2017velocity, zheng2021transporting} for density control, or \cite{milutinovic2006modeling, hamann2008framework} for general analytic frameworks, though they are significantly more difficult to treat both rigorously and from a learning perspective. Especially mean-field-based learning algorithms often remain restricted to competitive settings such as mean-field games \cite{lasry2007mean, huang2006large} by learning e.g. Nash \cite{yang2018mean, guo2019learning, subramanian2019reinforcement, perolat2021scaling, guo2022mf}, regularized \cite{anahtarci2020q, cui2021approximately} or correlated equilibria \cite{campi2022correlated, muller2022learning}. For instance, works such as \cite{gao2022energy} or \cite{shiri2019massive} investigate trajectory control of selfish UAV agents, while \cite{wang2022mean} considers formation flight in dense environments. Although selfish control problems are interesting for many applications, aligning selfish or local cost functions with a certain cooperative, global behavior can be difficult \cite{vsovsic2018reinforcement}. Solutions for cooperative joint objectives without necessity of manual cost function tuning are therefore of practical interest for artificially engineering swarm behaviors. 

In this work, we propose a discrete-time MFC-based swarm robotics framework that is conducive to powerful deep reinforcement learning (RL) techniques. Only very recently were MFC \cite{carmona2019model, gu2021mean, mondal2021approximation} and related histogram observations for MARL \cite{huttenrauch2019deep} proposed as a potential solution to cooperative scalable MARL, which could enable both the solution of otherwise intractable tasks as well as model-free application to swarms, adapting to environments and tasks. However, an eminent issue of MFC for robotic systems is violation of the MFC model due to physical collisions between robots. To solve this issue, we combine MFC with deep RL and collision avoidance algorithms. Here, collision avoidance algorithms could range from classical rule-based \cite{fiorini1998motion} over planning-based \cite{hamer2018fast} to learning-based approaches \cite{everett2021collision, ourari2022nearest}, and similarly for RL, see e.g. \cite{arulkumaran2017deep}. Importantly, our approach (i) is able to utilize advances in RL, circumventing MARL and solving otherwise difficult swarm problems without extensive manual and analytical design of algorithms, and (ii) closes the gap between mean-field models and reality, as collisions between agents violate the weak interaction principle of mean-field models and are usually to be avoided, e.g. in UAVs. As a result, our approach is highly practical, with the advantage of automatic design of swarm algorithms for swarm problems.

Our contribution can be summarized as follows: (i) We combine RL with MFC and collision avoidance algorithms for general task-driven control of robotic swarms; (ii) We give novel theoretical approximation guarantees of MFC in finite swarms as well as in the presence of additional collision avoidance maneuvers; (iii) We demonstrate in a variety of tasks that MFC outperforms state-of-the-art MARL, can be applied in a decentralized open-loop manner and avoids collisions, both in simulation and real UAV swarms. Overall, we provide a general framework for tractable swarm control that could be applied directly to swarms of UAVs.

\section{Swarm model}
In order to tractably describe a plethora of swarm tasks, we formulate a mean-field model where all agents are anonymous and it is sufficient to consider their distribution. 

\subsection{Finite swarm model}
Formally, we consider compact state and action spaces $\mathcal X, \mathcal U \subseteq \mathbb R^2$ (though our results are easily extended to $\mathbb R^3$) representing possible locations and movement choices of an agent. For any $N \in \mathbb N$, at each time $t = 0, 1, \ldots$, the states and actions of agent $i = 1, \ldots, N$ are denoted by $x^{i,N}_t$ and $u^{i,N}_t$. We denote by $\mathcal P(\mathcal X)$ the space of probability measures on $\mathcal X$, equipped with the topology of weak convergence. Define the empirical state distribution $\mu_t^N = \frac{1}{N} \sum_{i=1}^N \delta_{x_t^{i,N}} \in \mathcal P(\mathcal X)$, which represents all agents anonymously by their states. We consider policies $\pi = \{ \pi_t \}_{t \geq 0} \in \Pi$ from a space of policies $\Pi$ with shared Lipschitz constant, such that agents act on their location and the distribution of all agents, $\pi_t \colon \mathcal X \times \mathcal P(\mathcal X) \to \mathcal P(\mathcal U)$. The assumption of Lipschitz continuity is standard in the literature, includes e.g. neural networks \cite{mondal2021approximation, pasztor2021efficient, mondal2022on}, and may allow approximation of less regular policies.

Under a policy $\pi \in \Pi$, the finite swarm system shall evolve by sampling an initial state $x^{i,N}_{0} \sim \mu_0$ from an initial distribution $\mu_0$ of agents, and subsequently taking movement actions $u^{i,N}_t \sim \pi_t(x^{i,N}_t, \mu_t^N)$, resulting in new states $x^{i,N}_{t+1} = x^{i,N}_t + u^{i,N}_t + \epsilon^i_t$ for all agents $i$ with optional i.i.d. Gaussian noise $\epsilon^i_t \sim \mathcal N(0, \Sigma)$ and diagonal covariance matrix $\Sigma = \mathrm{diag}(\sigma_1^2, \sigma_2^2)$. In other words, each drone can move a distance limited to $\mathcal U$, up to some smoothing or inaccuracy $\epsilon^i_t$. In simulation, we further clip agent positions to stay inside $\mathcal X$. The objective is then given by an arbitrary function $r \colon \mathcal P(\mathcal X) \to \mathbb R$ of the spatial distribution of agents, giving rise to the infinite-horizon discounted objective
\begin{align}
    J^N(\pi) = \E \left[ \sum_{t=0}^{\infty} \gamma^t r(\mu^N_t) \right].
\end{align}

Since MARL can be difficult in the presence of many agents (see e.g. combinatorial nature in \cite{zhang2021multi}), we will formulate and verify a limiting infinite-agent system.

\subsection{Mean-field swarm model}
In the limit as $N \to \infty$, single agents become indiscernible and we need only model their distribution (mean-field) $\mu_t \in \mathcal P(\mathcal X)$. Starting at $\mu_0$, under policy $\pi \in \Pi$, deterministically
\begin{multline}
    \mu_{t+1} = T^{\pi_t}(\mu_t) \equiv T^{\pi_t(\cdot \mid \cdot, \mu_t)}(\mu_t) \\
    \coloneqq \iint \mathcal N(x + u, \sigma^2) \, \pi_t(\mathrm du \mid x, \mu_t) \, \mu_t(\mathrm dx)
\end{multline}
with shorthand of the deterministic mean-field transition operator $T^{\pi_t}(\mu_t) \equiv T^{\pi_t(\cdot \mid \cdot, \mu_t)}(\mu_t)$, $\pi_t(\cdot \mid \cdot, \mu_t) \in \mathcal P(\mathcal U)^{\mathcal X}$, giving way to the MFC problem with objective function
\begin{align}
    J(\pi) = \E \left[ \sum_{t=0}^{\infty} \gamma^t r(\mu_t) \right].
\end{align}

\begin{remark} \label{remark1}
A dependence of $r$ on joint state-action distributions in $\mathcal P(\mathcal X \times \mathcal U)$ can be modelled by splitting time steps into two and using the new state space $\mathcal X \cup (\mathcal X \times \mathcal U)$.
\end{remark}

For simplicity of analysis, we assume absence of common noise, leading to a deterministic mean-field limit, though in our experiments we also allow reactions to a random external environment. Under a mild continuity assumption, weaker than the common Lipschitz assumption in existing literature \cite{mondal2021approximation, mondal2022on}, we obtain rigorous approximation guarantees.

\begin{assumption} \label{ass}
The reward function $r$ is continuous.
\end{assumption}

By compactness of $\mathcal P(\mathcal X)$, $r$ is bounded. As long as $r$ is continuous, i.e. small changes in the agent distribution lead to small changes in reward, the MFC model is a good approximation for large swarms and its solution solves the finite agent system approximately optimally. As existing approximation properties still remain limited to finite $\mathcal X$, $\mathcal U$ \cite{mondal2021approximation, gu2021mean}, we give a brief, novel proof for compact spaces.

\begin{theorem} \label{thm:conv}
Under Assumption~\ref{ass}, at all times $t \in \mathcal T$, the empirical reward $r(\mu_t^N)$ converges weakly and uniformly to the limiting reward $r(\mu_t)$ as $N \to \infty$, i.e.
\begin{align}
    \sup_{\pi \in \Pi} \E \left[ \left| r(\mu^N_{t}) - r(\mu_{t}) \right| \right] \to 0.
\end{align}
\end{theorem}
\begin{proof}
We can metrize $\mathcal P(\mathcal X)$ via the metric $d(\mu, \nu) \coloneqq \sum_{m=1}^\infty 2^{-m} | \mu(f_m) - \nu(f_m) |$ for a sequence of continuous and bounded $f_m \colon \mathcal X \to \mathbb R$, $|f_m| \leq 1$ (cf. \cite[Theorem~6.6]{parthasarathy2005probability}).

Consider any (uniformly) equicontinuous set $\mathcal F \subseteq \mathbb R^{\mathcal P(\mathcal X)}$ of functions, i.e. there exists an increasing (concave, cf. \cite[p. 41]{devore1993constructive}) $\omega_{\mathcal F} \colon [0, \infty) \to [0, \infty)$ (modulus of continuity) such that $\omega_{\mathcal F}(x) \to 0$ when $x \to 0$ and $|f(\mu) - f(\nu)| \leq \omega_{\mathcal F}(d(\mu, \nu))$ for all $f \in \mathcal F$. We show inductively for $t \geq 0$ that
\begin{align} \label{eq:ind}
    \sup_{\pi \in \Pi} \sup_{f \in \mathcal F} \E \left[ \left| f(\mu^N_{t}) - f(\mu_{t}) \right| \right] \to 0,
\end{align}
which implies the desired property, since $r$ is uniformly continuous by compactness of $\mathcal P(\mathcal X)$ and Assumption~\ref{ass}.

At time $t=0$, the proof follows from the weak law of large numbers (LLN) argument (see \eqref{eq:first} and below). For the induction step,
\begin{align}
    &\sup_{\pi \in \Pi} \sup_{f \in \mathcal F} \E \left[ \left| f(\mu^N_{t+1}) - f(\mu_{t+1}) \right| \right] \\
    &\quad \leq \sup_{\pi \in \Pi} \E \left[ \omega_{\mathcal F}(d(\mu^N_{t+1}, T^{\pi_t}(\mu^N_t)) \right] \label{eq:first} \\
    &\qquad + \sup_{\pi \in \Pi} \sup_{f \in \mathcal F} \E \left[ \left| f(T^{\pi_t}(\mu^N_t)) - f(\mu_{t+1}) \right| \right] \label{eq:second}
\end{align}
where for the first term \eqref{eq:first}, by Jensen's inequality we obtain 
\begin{align*}
    &\E \left[ \omega_{\mathcal F}(d(\mu^N_{t+1}, T^{\pi_t}(\mu^N_t)) \right] \leq \omega_{\mathcal F} \left( \E \left[ d(\mu^N_{t+1}, T^{\pi_t}(\mu^N_t) \right] \right)
\end{align*}
for concave $\omega_{\mathcal F}$. Abbreviating $x^N_t \equiv \{x^{i,N}_t\}_{i\in[N]}$, we have
\begin{align*}
    &\E \left[ d(\mu^N_{t+1}, T^{\pi_t}(\mu^N_t) \right] \\
    &\quad = \sum_{m=1}^\infty 2^{-m} \E \left[ \left| \mu^N_{t+1}(f_m) - T^{\pi_t}(\mu^N_t)(f_m) \right| \right] \\
    &\quad \leq \sup_{m \geq 1} \E \left[ \E \left[ \left| \mu^N_{t+1}(f_m) - T^{\pi_t}(\mu^N_t)(f_m) \right| \innermid x^N_t \right] \right],
\end{align*}
where by the weak LLN argument, the squared term
\begin{align*}
    &\E \left[ \left| \mu^N_{t+1}(f_m) - T^{\pi_t}(\mu^N_t)(f_m) \right| \innermid x^N_t \right]^2 \\
    & \leq \E \left[ \left| \frac 1 N \sum_{i=1}^N \left( f_m(x^{i,N}_{t+1}) - \E \left[ f_m(x^{i,N}_{t+1}) \innermid x^N_t \right] \right) \right|^2 \innermid x^N_t \right] \\
    & = \frac{1}{N^2} \sum_{i=1}^N \E \left[ \left( f_m(x^{i,N}_{t+1}) - \E \left[ f_m(x^{i,N}_{t+1}) \innermid x^N_t \right] \right)^2 \innermid x^N_t \right] \\
    & \leq \frac{4}{N} \to 0
\end{align*}
since for any $f_m$, the cross-terms are zero and $|f_m| \leq 1$.

For the second term \eqref{eq:second}, by induction assumption we have
\begin{align*}
    &\sup_{\pi \in \Pi} \sup_{f \in \mathcal F} \E \left[ \left| f(T^{\pi_t}(\mu^N_t)) - f(\mu_{t+1}) \right| \right] \\
    &\quad \leq \sup_{\pi \in \Pi} \sup_{g \in \mathcal G} \E \left[ \left| g(\mu^N_t) - g(\mu_t) \right| \right] \to 0
\end{align*}
using $g = f \circ T^{\pi_t}$ and the corresponding class $\mathcal G$ of functions with modulus of continuity $\omega_{\mathcal G} \coloneqq \omega_{\mathcal F} \circ \omega_{T}$, where $\omega_T$ denotes the uniform modulus of continuity of $T^{\pi_t}$ by uniform Lipschitz continuity of $\pi \in \Pi$.
\end{proof}

\begin{figure}[b]
    \centering
    \includegraphics[width=0.8\linewidth]{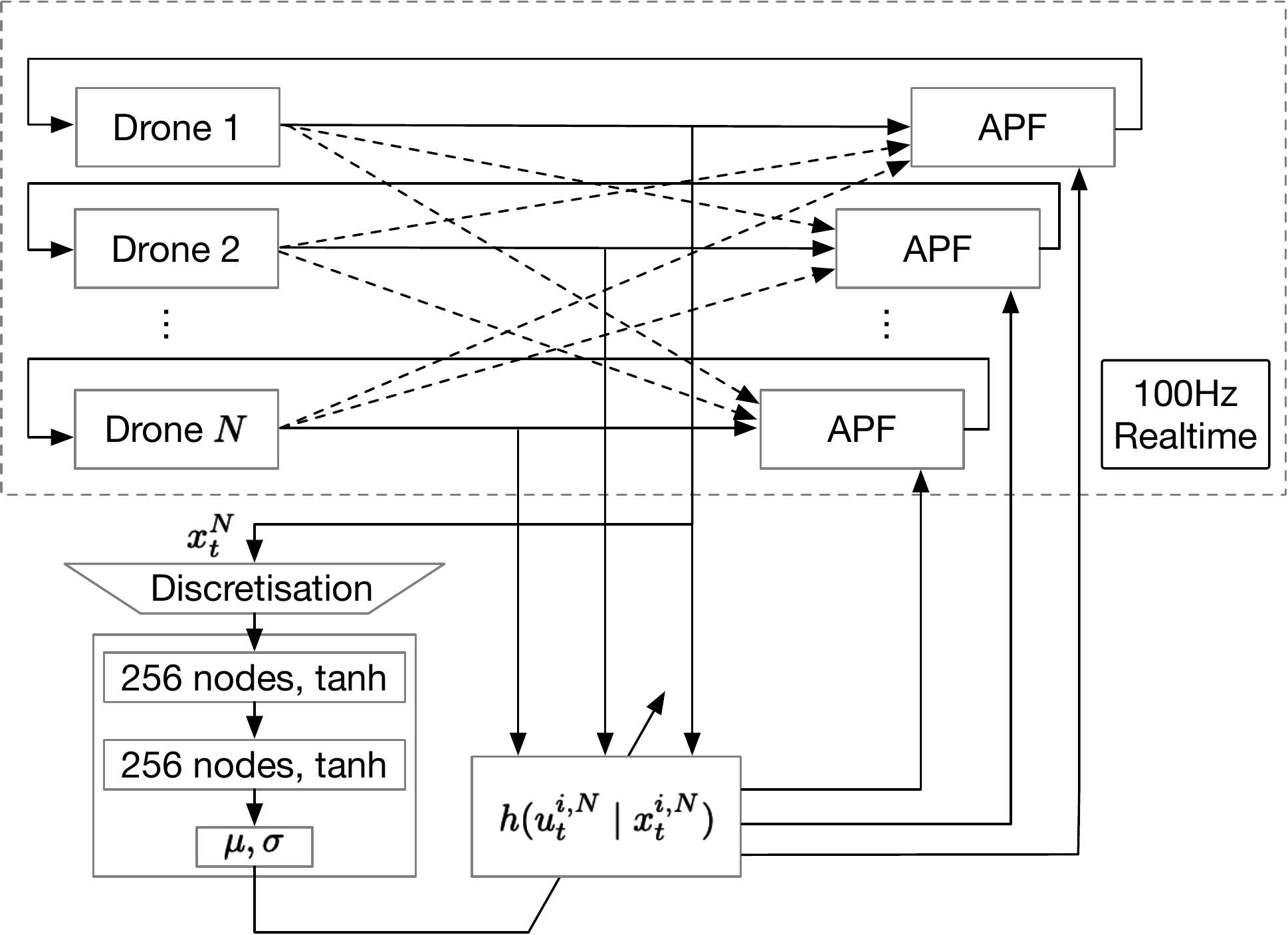}
    \caption{A hierarchical overview of our approach. The learned high-level mean-field control policy sends movement instructions to the (UAV) swarm, while each agent uses a real-time collision avoidance algorithm -- here artificial potential fields (APF)  -- to avoid collisions with others.}
    \label{fig:block_diagram}
\end{figure}

As a result, the MFC approach is a theoretically rigorous approach to approximately optimally solving large-scale swarm problems with complexity independent of $N$.

\begin{corollary}\label{cor:approx_sol_fin}
Under Assumption~\ref{ass}, an optimal solution $\pi^* \in \Pi$ to the MFC problem constitutes an $\varepsilon$-optimal solution to the finite swarm problem, where $\varepsilon \to 0$ as $N \to \infty$.
\end{corollary}
\begin{proof}
For any $\pi \in \Pi$ and $\varepsilon > 0$, we can choose $T$ such that $\sum_{t=T+1}^{\infty} \gamma^t \E \left[ \left| r(\mu^N_t) - r(\mu_t) \right| \right] \leq 2^{-T} \max_\mu 2 |r(\mu)| < \frac \varepsilon 4$, and for sufficiently large $N$ $\sum_{t=0}^{T} \gamma^t \E \left[ \left| r(\mu^N_t) - r(\mu_t) \right| \right] < \frac \varepsilon 4$ by Theorem~\ref{thm:conv}. Therefore, we have $J^N(\pi^*) - \max_{\pi \in \Pi} J^N(\pi) = \min_{\pi \in \Pi} (J^N(\pi^*) - J^N(\pi)) \geq \min_{\pi \in \Pi} (J^N(\pi^*) - J(\pi^*)) + \min_{\pi \in \Pi} (J(\pi^*) - J(\pi)) + \min_{\pi \in \Pi} (J(\pi) - J^N(\pi)) \geq - \frac \varepsilon 2 + 0 - \frac \varepsilon 2 = - \varepsilon$ by the prequel and optimality of $\pi^*$ in the MFC problem.
\end{proof}

\section{Methodology}
In order to remove the two remaining obstacles of (i) solving the MFC problem, and (ii) resolving the real-world gap of MFC for embodied agents, we combine MFC with arbitrary powerful RL and collision avoidance techniques. The overall hierarchical structure is found in Fig.~\ref{fig:block_diagram}. The MFC solution is learned via RL and gives high-level directions, which are realized by each agent while avoiding collisions.

\subsection{Reinforcement learning}
For the MFC problem, it is known that there exists an optimal stationary solution \cite[Theorem~19]{carmona2019model}, which may be found by solving the MFC Markov decision problem (MDP), a single-agent but infinite-dimensional RL problem with $\mathcal P(\mathcal X)$-valued states $\mu_t$ and $\mathcal P(\mathcal U)^{\mathcal X}$-valued actions $h_t$ evolving according to $\mu_{t+1} = T^h(\mu_t)$. To deal with the infinite dimensionality of $\mathcal P(\mathcal X)$ and $\mathcal P(\mathcal U)^{\mathcal X}$, we discretize $\mathcal X$ and use a binned histogram of $\mathcal P(\mathcal X)$ as in \cite{carmona2019model} with $M = 6^2 = 36$ bins by trading off between tractability (good training, low $M$) and performance (high $M$), while $\mathcal P(\mathcal U)$ is parametrized by Gaussians with means $\theta \in \mathcal U$ and diagonal covariances $\sigma_1, \sigma_2 \in (0, 0.25]$, of which the samples $u^i_t \sim h(\cdot \mid x^i_t) = \mathcal N(\theta, \mathrm{diag}(\sigma_1, \sigma_2))$ are clipped to $\mathcal U$. As exact computation of $\mu_t$ is difficult, we use the finite system with $N=300$ agents (though less works fine) and their empirical distribution analogous to particle filtering, which can be understood as directly learning on a large finite swarm.

We use the RLlib 1.13.0 implementation \cite{liang2018rllib} of proximal policy optimization (PPO) RL \cite{schulman2017proximal} and a diagonal Gaussian neural network policy with two hidden $\tanh$-layers of $256$ nodes, sampling clipped values in $[-1, 1]$ affinely transformed to $\theta, \sigma_1, \sigma_2$. Hyperparameters are printed in Table~\ref{tab:hyperparams}, of which sufficiently high minibatch sizes appeared most important.

\begin{table}
    \centering
    \caption{Hyperparameter configurations for PPO.}
    \label{tab:hyperparams}
    \begin{tabular}{@{}ccc@{}}
    \toprule
    Symbol     & Name          & Value     \\ \midrule
    $\gamma$ &   Discount factor &  $0.99$\\
    $\lambda$ &   GAE lambda &  $1$\\
    $\beta$ &   KL coefficient & $0.03$ \\
    $\epsilon$ &  Clip parameter & $0.2$ \\
    $l_{r}$ &   Learning rate & $0.00005$ \\
    $B_{b}$ &  Training batch size &  $4000$ \\
    $B_{m}$ &  Minibatch size &  $1000$ \\
    $T_b$ &  Updates per training batch & $5$ \\ \bottomrule
    \end{tabular}
\end{table}

\subsection{Collision avoidance subroutine}
A solution of the mean-field system does not directly translate into applicable real-world behavior, since the mean-field solution ignores physical constraints. While e.g. UAVs could fly at different heights, a general swarm algorithm should explicitly avoid collisions in order to guarantee suitability of the weakly-interacting MFC model. This is done by separating concerns, decomposing the issue into MFC plus sequences of collision-avoiding navigation subproblems between decision epochs. For example, we could choose $\mathcal U$ slightly smaller than the maximum speed range to allow for additional avoidance maneuvers. Then, assuming the time $\Delta t$ between two MFC decisions $t$ and $t+1$ is sufficiently long, and that agents have finer, direct control over their positions, a collision-avoiding navigation subroutine could approximately achieve the desired positions up to an error that becomes arbitrarily small with agent radius $r$.

For $N$ drones and agent radius $r$ we hence assume existence of such a subroutine $F$ which mildly perturbs all positions and their distribution $\mu^N_t$ at each time step and thereby achieves a collision-free mean field, which we write as $F(\mu^N_t)$, such that $\lVert x_t^i - x_t^j \rVert_2 > 2 r$ for all $i,j$. We further assume that $F$ is near-optimal, i.e. each drone's position is perturbed at most by a distance of $4 N r$. Indeed, this is possible for sufficiently small $r$, e.g. if $\mathcal X = [-m, m]^2$ for $m > 0$: At any $x \in \mathcal X$, on an arbitrary line of length greater $2m$ passing through $x$, we can always choose a position that is at most $4 N r$ away from $x$, as in the worst case all other $N-1$ drones are located on the line along which $F$ moves the drone and have a distance of slightly less than $4 r$ between each other. Under $F$, we can show that a collision-avoiding finite swarm of sufficiently many small agents is solved well by our approach.

\begin{theorem} \label{thm:coll}
Let $\pi \in \Pi$ be an optimal solution to the MFC problem, and let $F$ be the near-optimal collision avoidance subroutine as defined above. Then for each $\varepsilon > 0$ there exists an $N'$ such that for all $N \geq N'$ and agent radii $r_{N, \varepsilon}$, the solution $\pi$ gives an $\varepsilon$-optimal solution to the finite swarm problem with collision avoidance.
\end{theorem}
\begin{proof}
The definition of $F$ allows us to define new model dynamics with new random mean field variables denoted by $\mu'^N_t$, where we leave out the definition of each agent variable for brevity. For the new dynamics, at each time step $t$ we apply function $F$ to the current mean field $\mu'^N_t$ and underlying positions. Subsequently, the mean field $\mu'^N_{t+1}$ is obtained by applying the usual transition dynamics.

Now, we show via induction over $t$ that for all $t$,
\begin{align}\label{conv_thm_2}
    \sup_{\pi \in \Pi} \sup_{f \in \mathcal F}  \E \left[\left| f(\mu^N_{t}) - f(F(\mu'^N_t)) \right| \right] \to 0.
\end{align}
Analogous to the proof of Theorem~\ref{thm:conv}, the induction start follows from a weak LLN argument. For the induction step,
\begin{align}
    &\sup_{\pi \in \Pi} \sup_{f \in \mathcal F}  \E \left[\left| f(\mu^N_{t+1}) - f(F(\mu'^N_{t+1})) \right| \right] \nonumber \\
    &\quad \leq \sup_{\pi \in \Pi} \E \left[ \omega_{\mathcal F}(d(\mu^N_{t+1}, T^{\pi_t}(\mu^N_t)) \right] \nonumber\\
    &\quad + \sup_{\pi \in \Pi} \sup_{f \in \mathcal F}  \E \left[ \left| f(T^{\pi_t}(\mu^N_t)) - f(T^{\pi_{t}} (F(\mu'^N_t))) \right| \right] \nonumber\\
    &\quad + \sup_{\pi \in \Pi} \sup_{f \in \mathcal F} \E \left[ \left| f(T^{\pi_t}(F(\mu'^N_t))) - f(\mu'^N_{t + 1}) \right| \right] \nonumber\\  
    &\quad + \sup_{\pi \in \Pi} \sup_{f \in \mathcal F} \E \left[ \left| f(\mu'^N_{t+1}) - f(F(\mu'^N_{t + 1})) \right| \right] \label{eq:thm2_third}
\end{align}
where the first two summands converge to zero by arguments as in the proof of Theorem~\ref{thm:conv}. The third term converges to zero by a weak LLN argument while the forth summand is bounded  by $\omega_{\mathcal{F}} (4N r_{N, \varepsilon})$, see the explanation above. By choosing $r_{N, \varepsilon} = o (1/N)$, the last summand in \eqref{eq:thm2_third} converges to $0$. This concludes the induction.

For $\varepsilon$-optimality, we proceed as in Corollary~\ref{cor:approx_sol_fin} and obtain
\begin{align*}
    &\E \left[ \sum_{t=0}^{T-1} \gamma^t \left| r(\mu_t) - r(F(\mu'^N_t)) \right| \right]
    < \frac \varepsilon 2
\end{align*}
for $N$ large enough by applying statement \eqref{conv_thm_2}. The terms beyond $T-1$ can be bounded by $\varepsilon /2$ as in Corollary~\ref{cor:approx_sol_fin}.
\end{proof}

Hence, for a given allowed sub-optimality specification $\varepsilon$, we can find a number $N$ and size $r$ of drones such that solving the MFC problem is $\varepsilon$-optimal in the finite swarm system. In practice, this means that if we can use sufficiently many sufficiently small drones, MFC provides good solutions.

In this work, for simplicity we use artificial potential fields (APF) as in \cite{khatib1985real} with attractive velocity $F_d = 1.5 (\hat x_t^i - x_t^i)$ in simulation, where $\hat x_t^i$ denotes the MFC-based target position, and similarly repulsive velocity from agent $j$ on agent $i$, $F_{ji} = 1.5 c_{\mathrm{rep}} \cdot (\frac{1}{\lVert x_t^i - x_t^j \rVert_2} - 1) \cdot \frac{x_t^i - x_t^j}{\lVert x_t^i - x_t^j \rVert_2^3}$ whenever $\lVert x_t^i - x_t^j \rVert_2 \leq 1$ and zero otherwise, where $c_{\mathrm{rep}} > 0$ is a variable repulsion coefficient. However, we stress that other more advanced collision avoidance algorithms could be used.

\section{Experiments}
In this section, we verify the usefulness of MFC-based robotic swarm control experimentally.

\subsection{Problems}
We consider three problems of increasing complexity to demonstrate our approach. In the following, we consider uniform initial state distributions $\mu_0 = \mathrm{Unif}(\mathcal X)$ and let $\mathcal X = [-2, 2]^2$, allowing circular-constrained, noise-free movement, i.e. circular $\mathcal U$ such that $\lVert u_t^i \rVert_2 \leq 0.2$ with $\epsilon^i_t \equiv 0$.

\paragraph{Aggregation}
In the simple Aggregation or Rendezvous \cite{huttenrauch2019deep} problem, the goal of agents is to aggregate into a point while minimizing movement. Hence, we choose rewards $r(\nu_t) = \iint - \lVert x - \int x \, \nu_t(\mathrm dx, \mathrm du) \rVert_2 - 0.3 \lVert u \rVert_2 \, \nu_t(\mathrm dx, \mathrm du)$ for joint state-actions $\nu_t = \mu_t \otimes h_t \in \mathcal P(\mathcal X \times \mathcal U)$ (see Remark~\ref{remark1}).

\paragraph{Formation}
In the Formation problem, the goal is to achieve an anonymous formation flight of large swarms, i.e. matching the distribution of agent positions with a given distribution -- e.g. for providing coverage for surveillance or communication. The rewards are given by the Wasserstein distance $r(\mu_t) = \inf_{X,Y \colon \mathcal L(X) = \mu_t, \mathcal L(Y) = \mu^*} \mathbb E \left[ \lVert X - Y \rVert_2 \right]$ \cite{villani2009optimal} between agent distribution $\mu_t$ and e.g. a Gaussian mixture $\mu^* = \frac 1 2 \mathcal N(e_1, \mathrm{diag}(0.05, 0.05)) + \frac 1 2 \mathcal N(-e_1, \mathrm{diag}(0.05, 0.05))$ with unit vector $e_1$, computed via the empirical Wasserstein distance between agents and $300$ samples of $\mu^*$. In principle, it is also possible to add movement costs as in Aggregation.

\paragraph{Task allocation}
Lastly, we formulate a problem with stochasticity even in the limit. Consider randomly generated, spatially localized tasks such as providing a UAV-based communication uplink, or emergency operations for clearing rubble and firefighting. We add spatially localized tasks to the model which are observed via an additional histogram of task locations. Here, in each time step, $N_t = \mathrm{Pois}(0.4)$ tasks $l$ arrive at uniformly random points $x^l \in \mathcal X$, up to a maximum of $5$ total tasks. Each task $l$ begins with length $L_t = 10$ and at each time step is processed abstractly by proximity of nearby agents according to $L^l_{t+1} = L^l_t - \Delta L^l(\mu_t)$, $\Delta L^l(\mu_t) \coloneqq \min(1, \int (1 - 2 \lVert x - x^l \rVert_2) \mathbf 1_{\lVert x - x^l \rVert_2 \leq 0.5} \, \mu_t(\mathrm dx)$, until it is fully processed and disappears. The reward is defined by the processed task lengths $r(\mu_t) = \sum_l \Delta L^l(\mu_t)$.

\begin{figure}[tb]
    \centering
    \includegraphics[width=0.95\linewidth]{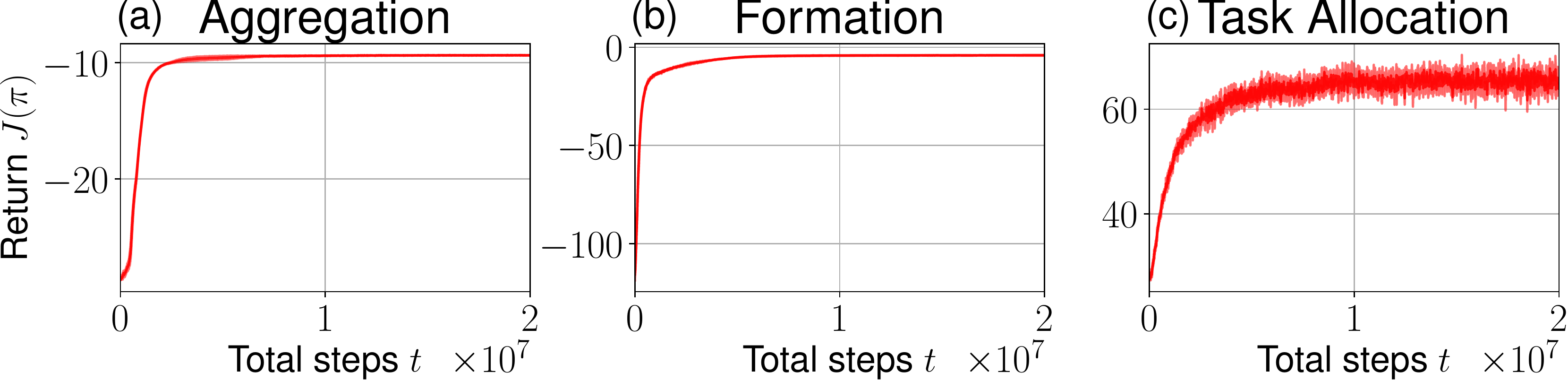}
    \caption{Training curves of the MFC algorithm trained on $N=300$, plotting the average achieved objective over time steps taken, together with its standard deviation over $3$ seeds. The MFC approach leads to very stable learning results for all of our considered problems. (a): Aggregation; (b): Formation; (c): Task Allocation.}
    \label{fig:training_mfc}
\end{figure}

\begin{figure}[b]
    \centering
    \includegraphics[width=0.95\linewidth]{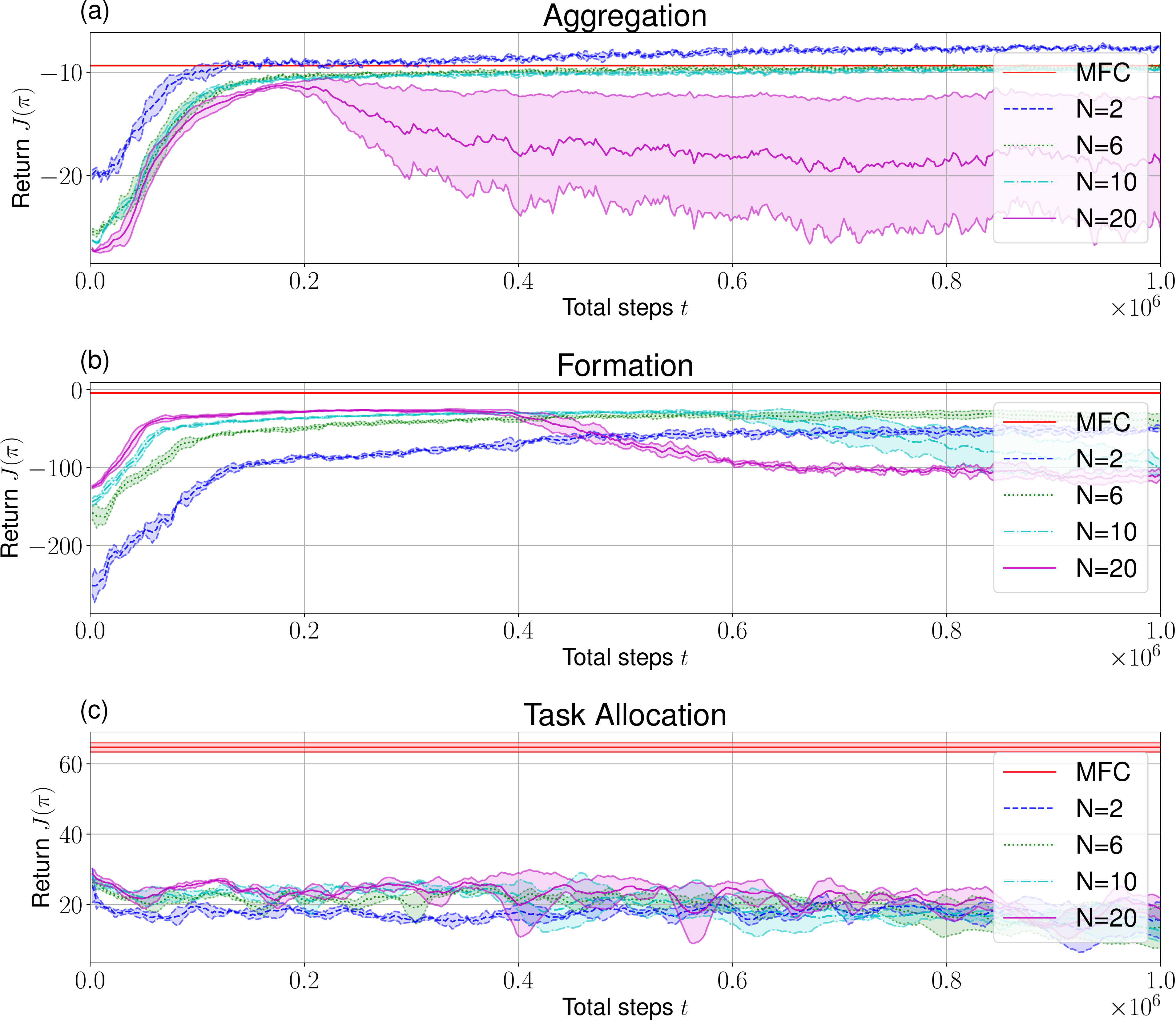}
    \caption{Training curves of the $N$-agent MARL algorithm, plotting the average achieved objective over time steps taken, together with its standard deviation over $3$ seeds and compared to final MFC performance (red, $N=300$). (a): In the simple Aggregation task, MARL and MFC are comparable for few agents, but MARL fails for many agents. (b-c): In more complex scenarios, MFC converges to a better solution than common MARL.}
    \label{fig:training_marl}
\end{figure}

\begin{figure}[t]
    \centering
    \includegraphics[width=0.95\linewidth]{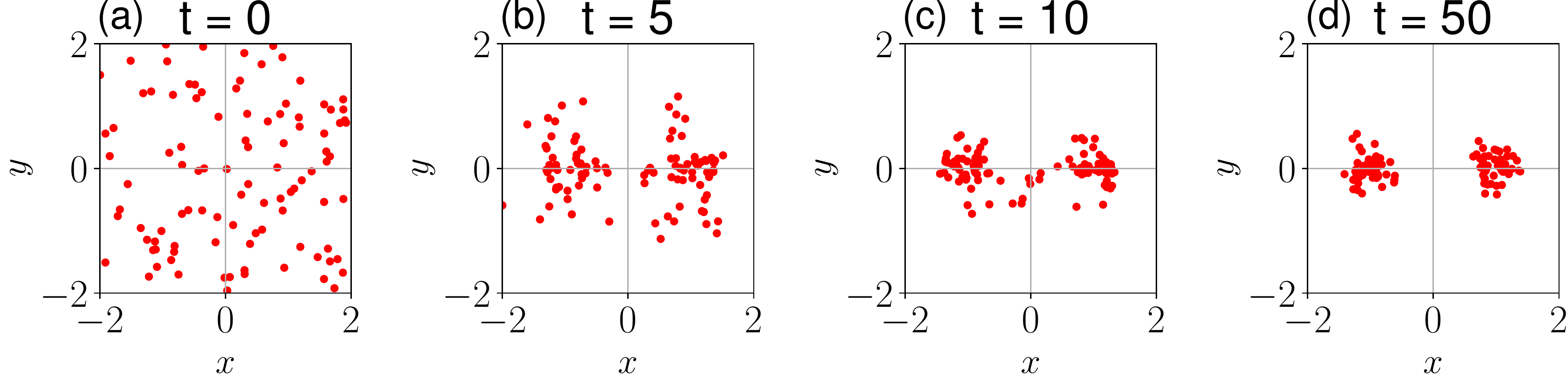}
    \caption{One sample run of the MFC solution to the Formation problem, applied to a system with $N=100$ agents and plotted at times $t \in \{0, 5, 10, 50\}$. Agents successfully form a mixture of two Gaussians.}
    \label{fig:formation}
\end{figure}

\begin{figure}[b]
    \centering
    \includegraphics[width=0.95\linewidth]{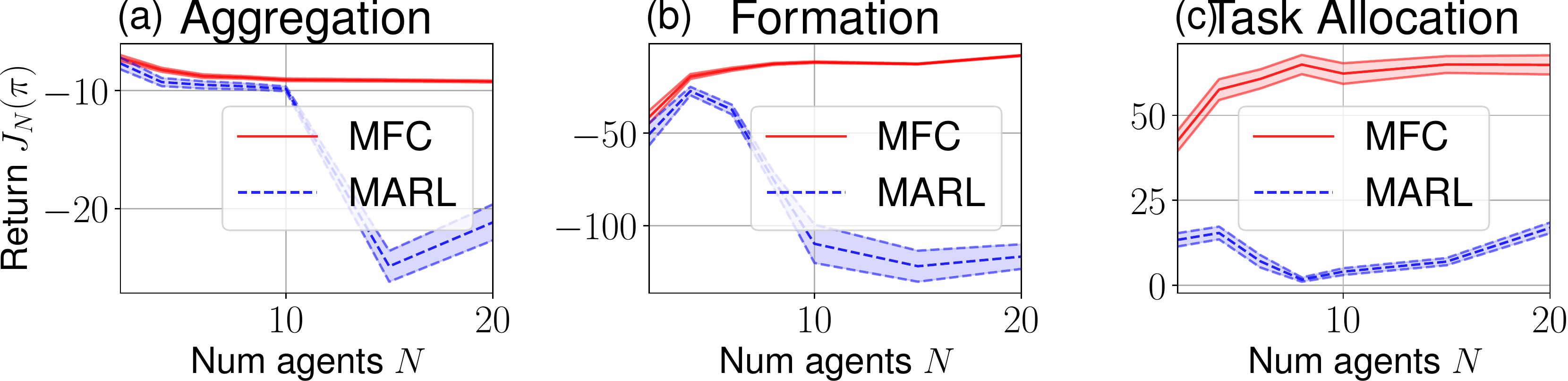}
    \caption{Comparison of achieved objectives in the finite swarm system of MFC and MARL solutions over $100$ sample episodes, with $95\%$ confidence interval (shaded). The MFC algorithm quickly converges to the deterministic, limiting mean-field objective as $N$ becomes large. In simple scenarios such as Aggregation (a), MARL outperforms MFC in the finite system, while in more complex scenarios (b-c), MFC outperforms MARL (at end of training). }
    \label{fig:finite_compare}
\end{figure}

\begin{figure*}[tb]
    \centering
    \includegraphics[width=0.32\linewidth]{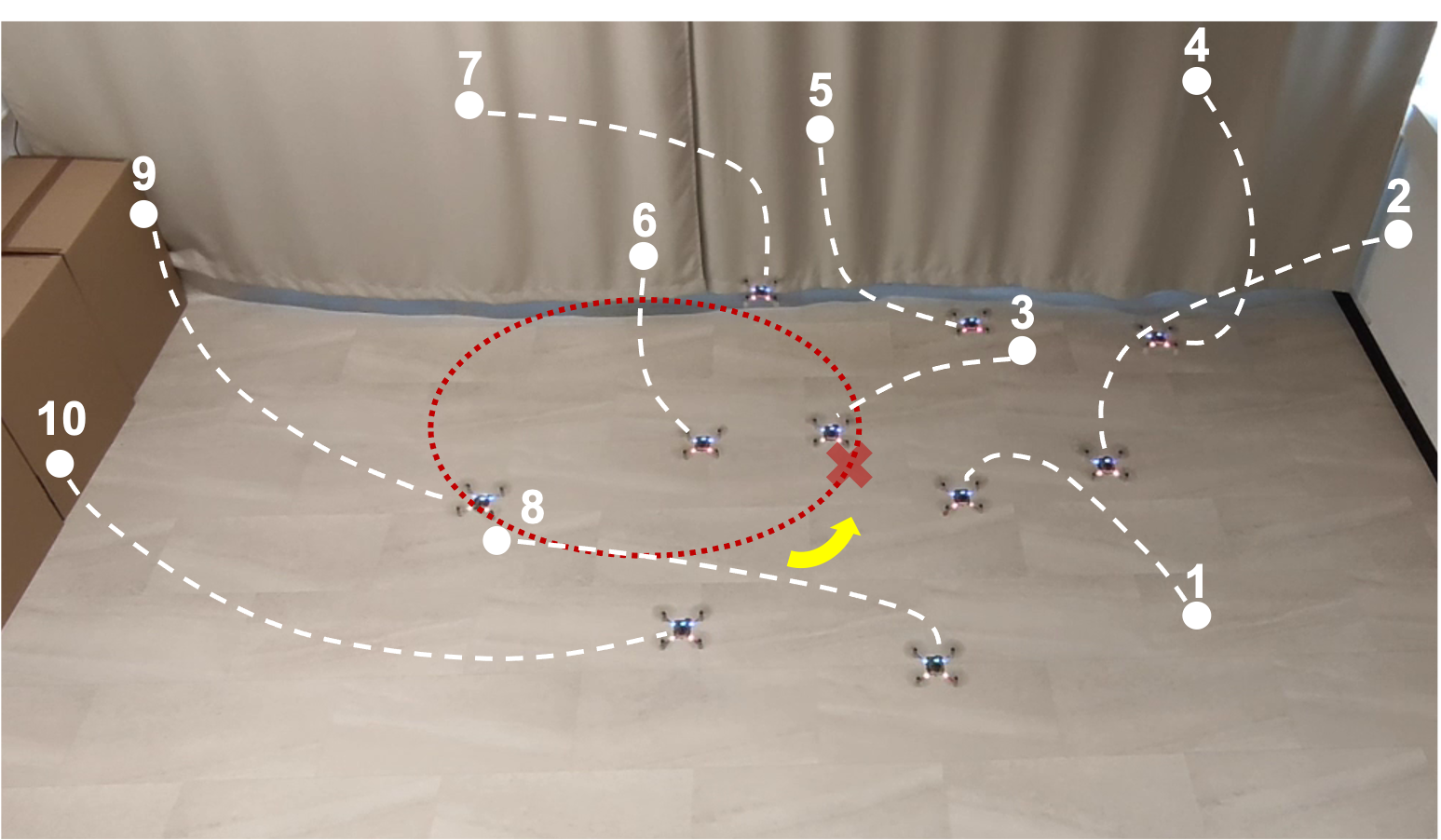} \includegraphics[width=0.32\linewidth]{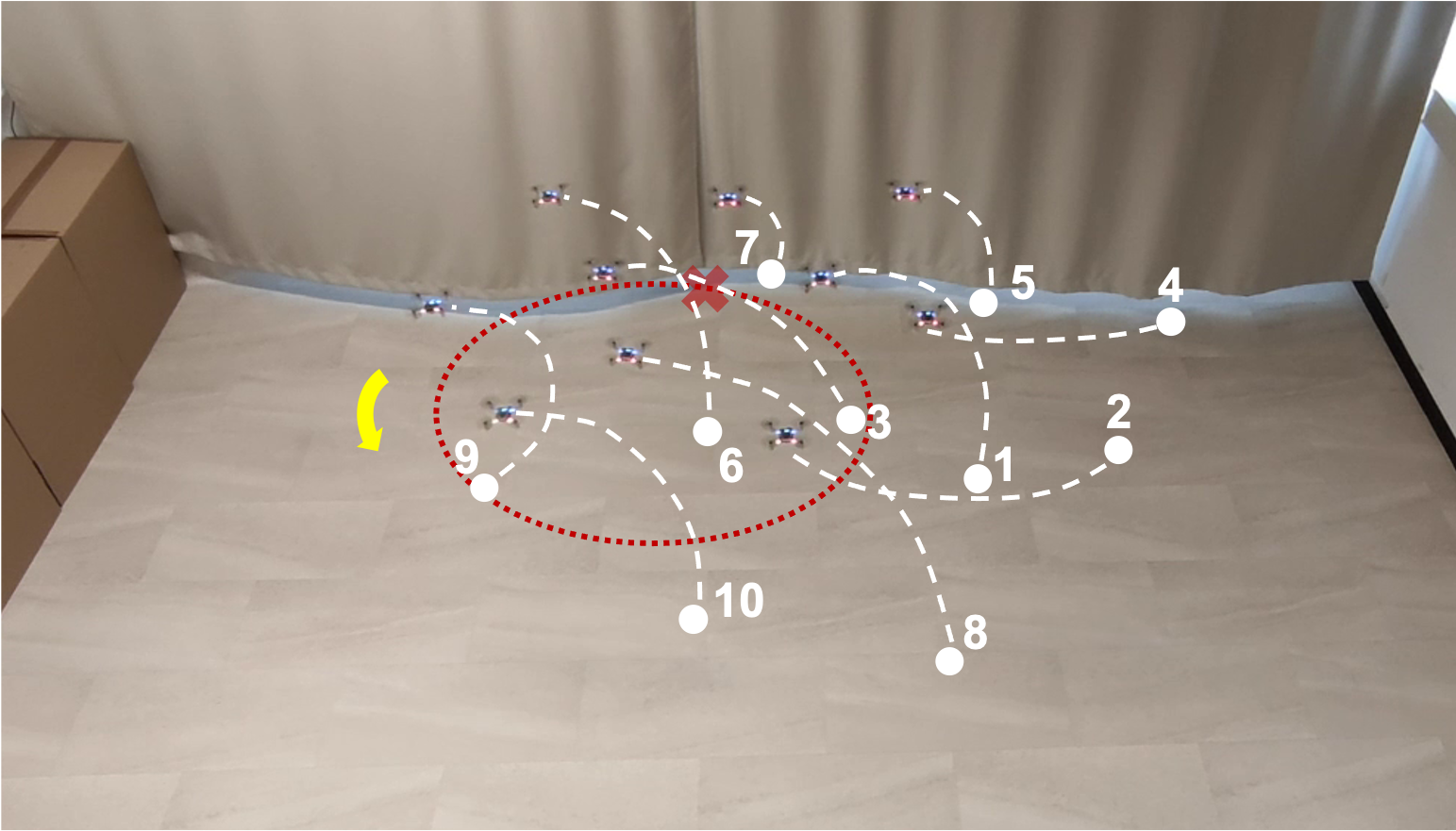} \includegraphics[width=0.32\linewidth]{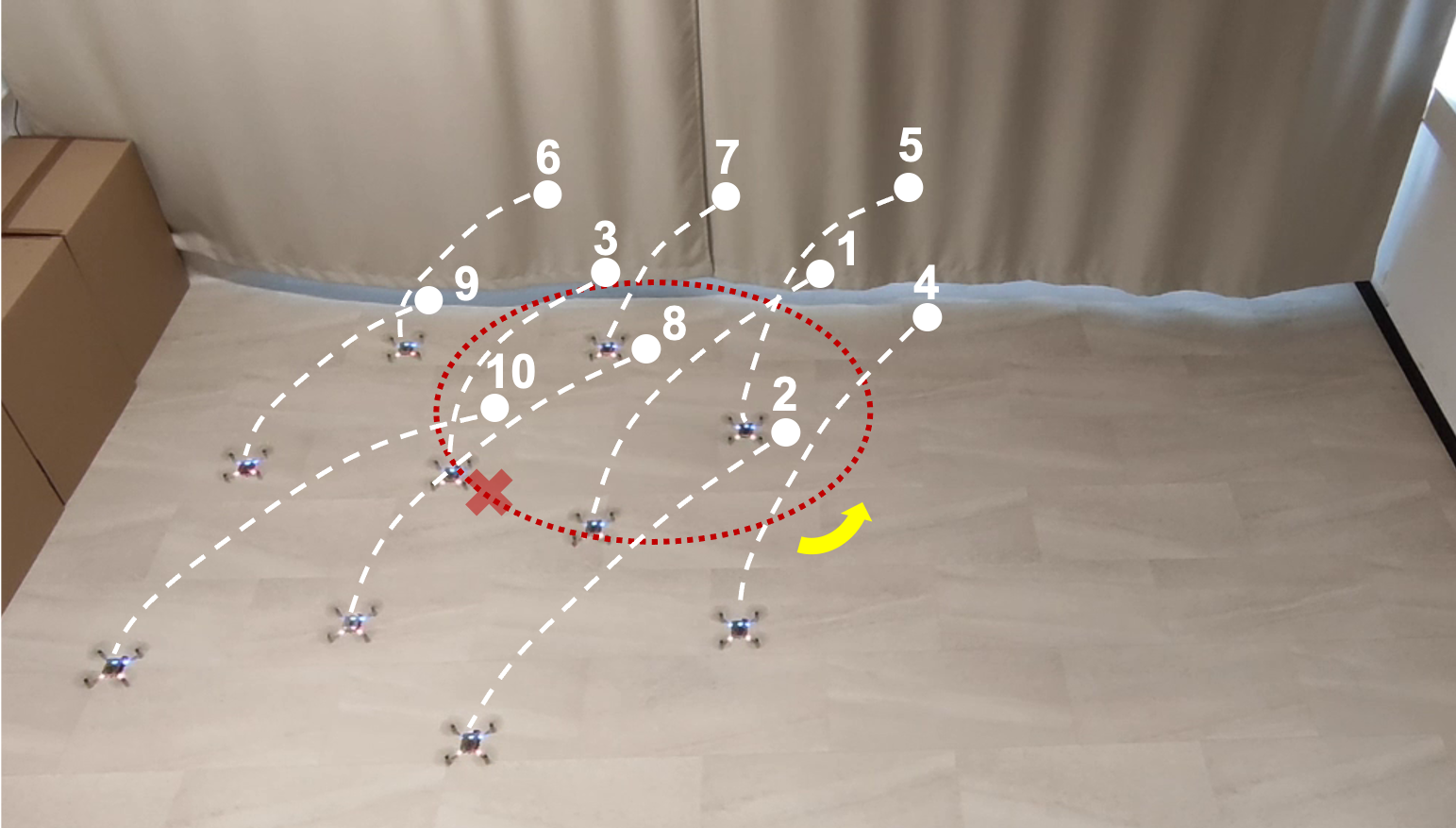}
    \caption{Real world coverage experiment with a swarm of 10 Crazyflie nano-quadcopters and a variant of the Formation problem where agents track a single time-varying Gaussian distribution (current center of Gaussian shown as red cross), moving counter-clockwise on a circle (red dotted line). The drones successfully track the time-varying Gaussian distribution using the open-loop control policy without collision. Time progresses from left to right.}
    \label{fig:real}
\end{figure*}

\subsection{Experimental results}
In the following, we show results demonstrating the power of our MFC framework for task-driven swarm control, namely their theoretical and numerical advantage over standard MARL, the potential for decentralized open-loop control, and the influence of collision avoidance on optimality, both in simulation and in the real world.

\paragraph{Training results}
In our implementation, each training episode consists of $50$, $100$ and $200$ time steps for the Aggregation, Formation and Task Allocation problems respectively, of which the average sum of rewards will constitute the return values shown in the following figures. As can be seen in Fig.~\ref{fig:training_mfc}, the learning curve of PPO in the MFC problem is smoothly increasing as expected, since the MFC MDP leads to a single-agent problem solvable via standard RL with better understood theory than MARL, e.g. \cite{schulman2015trust}.

In contrast, state-of-the-art MARL techniques miss theoretical guarantees. We compare to PPO with parameter-sharing \cite{gupta2017cooperative} and independent learning \cite{tan1993multi}, which has repeatedly achieved state-of-the-art performance in benchmarks \cite{de2020independent, yu2021surprising, papoudakis2021benchmarking, fu2022revisiting} and remains applicable to arbitrary numbers of homogeneous agents. For comparability, we use the same architecture and implementation as in our MFC experiments, outputting parameters of a Gaussian over actions. Each agent simply observes the same information plus the agent's own position. 

As seen in Fig.~\ref{fig:training_marl}, MARL works well in the very simple Aggregation task, but becomes increasingly unstable for many agents, especially in the more complex Formation scenario, finally failing entirely in Task Allocation due to non-stationarity of learning \cite{zhang2021multi}. Although MARL could work for other hyperparameter configurations, it shows that standard MARL can suffer from worse stability than single-agent RL in even high-dimensional single-agent MFC MDPs, congruent with the outstanding issue of theoretical MARL convergence guarantees \cite{zhang2021multi}. As seen exemplarily for the Formation problem in Fig.~\ref{fig:formation} and a variation of the problem with real drones (later) in Fig.~\ref{fig:real}, the MFC solution successfully achieves the desired mixture of Gaussian formation of agents. In Fig.~\ref{fig:finite_compare}, it can be seen that (i) the MFC solution outperforms MARL, and (ii) the MFC solution quickly converges to the limiting deterministic objective in Fig.~\ref{fig:training_mfc} as $N$ grows large, verifying the MFC approximation properties in Theorem~\ref{thm:conv}.

\begin{figure}[tb]
    \centering
    \includegraphics[width=0.95\linewidth]{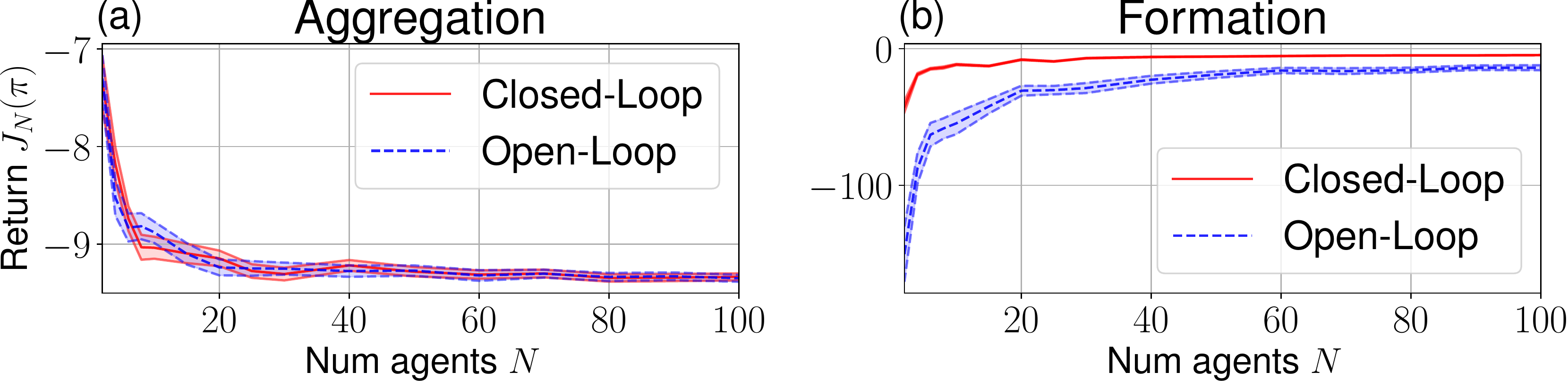}
    \caption{Comparison of mean objectives in the finite swarm system of closed-loop and open-loop MFC over $100$ sample episodes, with $95\%$ confidence interval (shaded). In Aggregation (a), little difference can be seen between the closed-loop and open-loop performance. In Formation (b), the open-loop policy is unable to react to stochastic initialization effects of finite swarm size, only approaching optimality in the large swarm limit.}
    \label{fig:open-loop}
\end{figure}

\begin{figure}[b]
    \centering
    \includegraphics[width=0.95\linewidth]{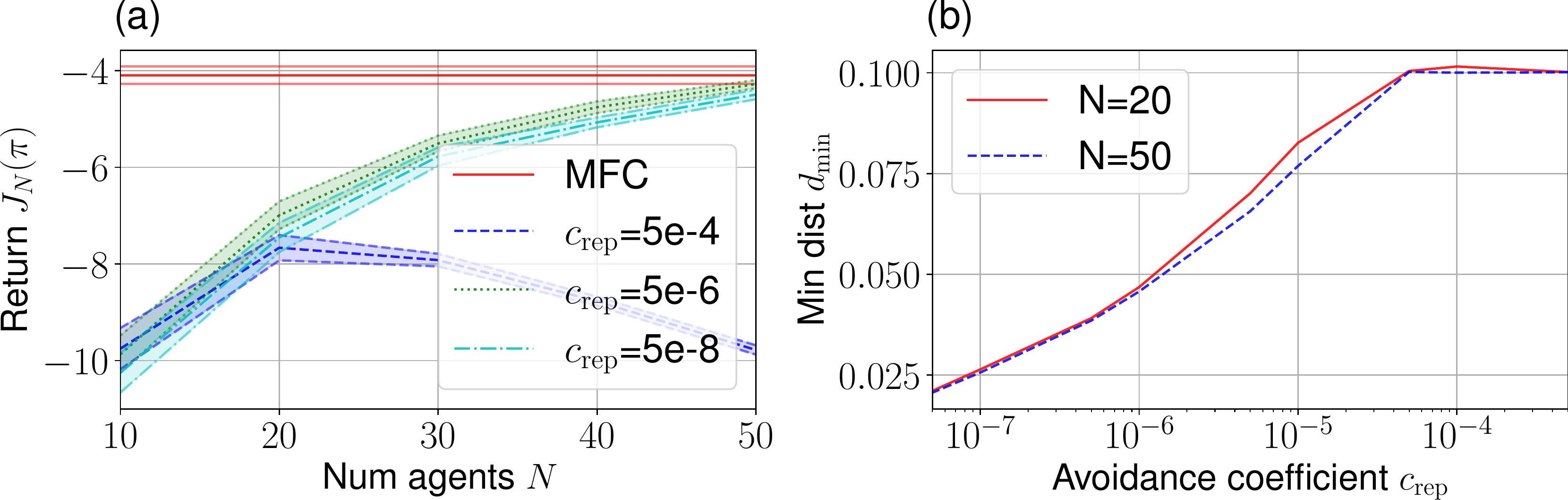}
    \caption{Comparison of results in the finite swarm system of MFC solution with collision avoidance for various collision avoidance coefficients $c_{\mathrm{rep}}$ in the Formation problem. (a): Mean objectives averaged over $100$ sample episodes, with $95\%$ confidence interval (shaded). (b): Minimum occurring inter-agent distance over $100$ sample episodes.}
    \label{fig:coll_avoid}
\end{figure}

\paragraph{Decentralized open-loop control}
In the absence of global information, it makes sense for large swarms to let agents act stochastically and independently, especially since agents are interchangeable and anonymous. For this purpose, as long as the limiting MFC is deterministic (e.g. Aggregation and Formation), we can compute an optimal open-loop control sequence $h_0, h_1, \ldots$ of MFC actions $h_t \in \mathcal P(\mathcal U)^{\mathcal X}$ for a given starting $\mu_0$, and apply $h_t$ to each agent. This results in both open-loop and decentralized control, as each agent moves depending on its own local position only. As expected by determinism of MFC, in Fig.~\ref{fig:open-loop} we observe that the open-loop performance becomes practically indistinguishable from the closed-loop performance in Aggregation, as well as approaches it in Formation for sufficiently large swarms. We note that at least for finite spaces, very recently a similar decentralization result was also rigorously shown \cite{mondal2022on}.

\paragraph{Influence of collision avoidance}
For MFC with collision avoidance, we simulate $\Delta t = 2$ and $100$ explicit Euler steps of length $0.02$ between each decision epoch $t$. Furthermore, to avoid bad initializations, we resample initial states until the minimal inter-agent distance is above $0.1$. As seen in Fig.~\ref{fig:coll_avoid}, the minimal inter-agent distance is easily tuned via $c_{\mathrm{rep}}$, rising up to the initialization distance $0.1$. We find that for strong collision avoidance, the performance deteriorates in the presence of many agents, whereas for smaller collision avoidance coefficients the performance approaches the mean-field limit, verifying Theorem~\ref{thm:coll}.

\paragraph{Illustrative real-life experiment}
Lastly, we show the results of applying a variant of the Formation task -- tracking a single time-variant Gaussian moving on a circle -- to a real fleet of Crazyflie quadcopters \cite{giernacki2017crazyflie}, each peer-to-peer-broadcasting only their local Lighthouse-based state estimates \cite{greiff2019performance}. Here, we use the aforementioned decentralized open-loop control and APF-based collision avoidance. Although our experiments remain small-scale due to space constraints and downwash effects, we nonetheless show that our approach works in practice and can be applied to even small swarm sizes. In the future, we imagine similar approaches to be scaled up to larger fleets. As can be seen in Fig.~\ref{fig:real}, the agents successfully track the formation without colliding.

\section{Conclusion}
In this work, we have proposed a scalable task-driven approach to robotic swarm control that allows for model-free solution of swarm tasks while remaining applicable in practice by using deep RL, MFC and collision avoidance. Our approach is hierarchical, in principle allowing to profit from any state-of-the-art RL and collision avoidance techniques. Our work is a step towards general toolchains for robotic swarm control, which yet remain part of active research \cite{schranz2020swarm}. We have solved part of the limitations of mean-field theory for embodied agents by integrating collision avoidance into the toolchain, but more work on more sparsely interacting mean-field models may be necessary, e.g. for UAV-based communication with strongly neighbor-dependent interaction, by incorporating graph structure \cite{caines2019graphon, cui2022learning, duan2022prevalence}. Extensions to non-linear dynamics and dynamical constraints may be fruitful. Lastly, while our Gaussian parametrization of $\mathcal P(\mathcal U)$ is efficient, the state discretization still suffers from a curse of dimensionality, as the number of bins rises quickly with fineness of discretization, which is state of the art \cite{carmona2019model, gu2021mean} and could be supplemented e.g. by visual techniques \cite{perrin2021generalization}.

\bibliographystyle{IEEEtran}
\bibliography{main}

\begin{thebibliography}{10}
\providecommand{\url}[1]{#1}
\csname url@samestyle\endcsname
\providecommand{\newblock}{\relax}
\providecommand{\bibinfo}[2]{#2}
\providecommand{\BIBentrySTDinterwordspacing}{\spaceskip=0pt\relax}
\providecommand{\BIBentryALTinterwordstretchfactor}{4}
\providecommand{\BIBentryALTinterwordspacing}{\spaceskip=\fontdimen2\font plus
\BIBentryALTinterwordstretchfactor\fontdimen3\font minus
  \fontdimen4\font\relax}
\providecommand{\BIBforeignlanguage}[2]{{%
\expandafter\ifx\csname l@#1\endcsname\relax
\typeout{** WARNING: IEEEtran.bst: No hyphenation pattern has been}%
\typeout{** loaded for the language `#1'. Using the pattern for}%
\typeout{** the default language instead.}%
\else
\language=\csname l@#1\endcsname
\fi
#2}}
\providecommand{\BIBdecl}{\relax}
\BIBdecl

\bibitem{brambilla2013swarm}
M.~Brambilla, E.~Ferrante, M.~Birattari, and M.~Dorigo, ``Swarm robotics: a
  review from the swarm engineering perspective,'' \emph{Swarm Intell.},
  vol.~7, no.~1, pp. 1--41, 2013.

\bibitem{schranz2020swarm}
M.~Schranz, M.~Umlauft, M.~Sende, and W.~Elmenreich, ``Swarm robotic behaviors
  and current applications,'' \emph{Front. Robot. AI}, vol.~7, p.~36, 2020.

\bibitem{chung2018survey}
S.-J. Chung, A.~A. Paranjape, P.~Dames, S.~Shen, and V.~Kumar, ``A survey on
  aerial swarm robotics,'' \emph{IEEE Trans. Robot.}, vol.~34, no.~4, pp.
  837--855, 2018.

\bibitem{dorigo2020reflections}
M.~Dorigo, G.~Theraulaz, and V.~Trianni, ``Reflections on the future of swarm
  robotics,'' \emph{Science Robotics}, vol.~5, no.~49, p. eabe4385, 2020.

\bibitem{correll2006system}
N.~Correll and A.~Martinoli, ``System identification of self-organizing robotic
  swarms,'' in \emph{Distributed Autonomous Robotic Systems 7}.\hskip 1em plus
  0.5em minus 0.4em\relax Springer, 2006, pp. 31--40.

\bibitem{tuci2018cooperative}
E.~Tuci, M.~H. Alkilabi, and O.~Akanyeti, ``Cooperative object transport in
  multi-robot systems: A review of the state-of-the-art,'' \emph{Front. Robot.
  AI}, vol.~5, p.~59, 2018.

\bibitem{grobeta2008evolution}
R.~Gross and M.~Dorigo, ``Evolution of solitary and group transport behaviors
  for autonomous robots capable of self-assembling,'' \emph{Adapt. Behav.},
  vol.~16, no.~5, pp. 285--305, 2008.

\bibitem{gross2009towards}
------, ``Towards group transport by swarms of robots,'' \emph{Int. J.
  Bio-Inspired Comput.}, vol.~1, no. 1/2, pp. 1--13, 2009.

\bibitem{albani2018dynamic}
D.~Albani, T.~Manoni, D.~Nardi, and V.~Trianni, ``Dynamic {UAV} swarm
  deployment for non-uniform coverage,'' in \emph{Proc. AAMAS}, 2018, pp.
  523--531.

\bibitem{xing2019offense}
D.~Xing, Z.~Zhen, and H.~Gong, ``Offense--defense confrontation decision making
  for dynamic {UAV} swarm versus {UAV} swarm,'' \emph{Proc. Inst. Mech. Eng.
  G}, vol. 233, no.~15, pp. 5689--5702, 2019.

\bibitem{vincent2004framework}
P.~Vincent and I.~Rubin, ``A framework and analysis for cooperative search
  using {UAV} swarms,'' in \emph{Proc. ACM Symp. Appl. Comput.}, 2004, pp.
  79--86.

\bibitem{zhang2021multi}
K.~Zhang, Z.~Yang, and T.~Ba{\c{s}}ar, ``Multi-agent reinforcement learning: A
  selective overview of theories and algorithms,'' \emph{Handbook of
  Reinforcement Learning and Control}, pp. 321--384, 2021.

\bibitem{berner2019dota}
\BIBentryALTinterwordspacing
C.~Berner, G.~Brockman, B.~Chan, V.~Cheung, P.~D{\k{e}}biak, C.~Dennison,
  D.~Farhi, Q.~Fischer, S.~Hashme, C.~Hesse \emph{et~al.}, ``Dota 2 with large
  scale deep reinforcement learning,'' \emph{arXiv:1912.06680}, 2019. [Online].
  Available: \url{https://arxiv.org/abs/1912.06680}
\BIBentrySTDinterwordspacing

\bibitem{al2015application}
H.~A. Al-Rawi, M.~A. Ng, and K.-L.~A. Yau, ``Application of reinforcement
  learning to routing in distributed wireless networks: A review,''
  \emph{Artif. Intell. Rev.}, vol.~43, no.~3, pp. 381--416, 2015.

\bibitem{chu2019multi}
T.~Chu, J.~Wang, L.~Codec{\`a}, and Z.~Li, ``Multi-agent deep reinforcement
  learning for large-scale traffic signal control,'' \emph{IEEE Trans. Transp.
  Syst.}, vol.~21, no.~3, pp. 1086--1095, 2019.

\bibitem{hamann2018swarm}
H.~Hamann, \emph{Swarm Robotics: A Formal Approach}.\hskip 1em plus 0.5em minus
  0.4em\relax Springer, 2018.

\bibitem{elamvazhuthi2019mean}
K.~Elamvazhuthi and S.~Berman, ``Mean-field models in swarm robotics: A
  survey,'' \emph{Bioinspir. Biomim.}, vol.~15, no.~1, p. 015001, 2019.

\bibitem{bensoussan2013mean}
A.~Bensoussan, J.~Frehse, P.~Yam \emph{et~al.}, \emph{Mean field games and mean
  field type control theory}.\hskip 1em plus 0.5em minus 0.4em\relax Springer,
  2013, vol. 101.

\bibitem{elamvazhuthi2017mean}
K.~Elamvazhuthi, M.~Kawski, S.~Biswal, V.~Deshmukh, and S.~Berman, ``Mean-field
  controllability and decentralized stabilization of markov chains,'' in
  \emph{Proc. IEEE CDC}, 2017, pp. 3131--3137.

\bibitem{deshmukh2018mean}
V.~Deshmukh, K.~Elamvazhuthi, S.~Biswal, Z.~Kakish, and S.~Berman, ``Mean-field
  stabilization of markov chain models for robotic swarms: Computational
  approaches and experimental results,'' \emph{IEEE Robot. Autom. Lett.},
  vol.~3, no.~3, pp. 1985--1992, 2018.

\bibitem{elamvazhuthi2018mean}
K.~Elamvazhuthi, S.~Biswal, and S.~Berman, ``Mean-field stabilization of
  robotic swarms to probability distributions with disconnected supports,'' in
  \emph{Proc. IEEE ACC}, 2018, pp. 885--892.

\bibitem{mayya2018localization}
S.~Mayya, P.~Pierpaoli, G.~Nair, and M.~Egerstedt, ``Localization in densely
  packed swarms using interrobot collisions as a sensing modality,'' \emph{IEEE
  Trans. Robot.}, vol.~35, no.~1, pp. 21--34, 2018.

\bibitem{mayya2019closed}
S.~Mayya, S.~Wilson, and M.~Egerstedt, ``Closed-loop task allocation in robot
  swarms using inter-robot encounters,'' \emph{Swarm Intell.}, vol.~13, no.~2,
  pp. 115--143, 2019.

\bibitem{lerman2001macroscopic}
K.~Lerman, A.~Galstyan, A.~Martinoli, and A.~Ijspeert, ``A macroscopic
  analytical model of collaboration in distributed robotic systems,''
  \emph{Artif. Life}, vol.~7, no.~4, pp. 375--393, 2001.

\bibitem{eren2017velocity}
U.~Eren and B.~A{\c{c}}{\i}kme{\c{s}}e, ``Velocity field generation for density
  control of swarms using heat equation and smoothing kernels,''
  \emph{IFAC-PapersOnLine}, vol.~50, no.~1, pp. 9405--9411, 2017.

\bibitem{zheng2021transporting}
T.~Zheng, Q.~Han, and H.~Lin, ``Transporting robotic swarms via mean-field
  feedback control,'' \emph{IEEE Trans. Autom. Control}, 2021.

\bibitem{milutinovic2006modeling}
D.~Milutinovi{\'c} and P.~Lima, ``Modeling and optimal centralized control of a
  large-size robotic population,'' \emph{IEEE Trans. Robot.}, vol.~22, no.~6,
  pp. 1280--1285, 2006.

\bibitem{hamann2008framework}
H.~Hamann and H.~W{\"o}rn, ``A framework of space--time continuous models for
  algorithm design in swarm robotics,'' \emph{Swarm Intell.}, vol.~2, no.~2,
  pp. 209--239, 2008.

\bibitem{lasry2007mean}
J.-M. Lasry and P.-L. Lions, ``Mean field games,'' \emph{Japanese J. Math.},
  vol.~2, no.~1, pp. 229--260, 2007.

\bibitem{huang2006large}
M.~Huang, R.~P. Malham{\'e}, P.~E. Caines \emph{et~al.}, ``Large population
  stochastic dynamic games: closed-loop mckean-vlasov systems and the nash
  certainty equivalence principle,'' \emph{Commun. Inf. Syst.}, vol.~6, no.~3,
  pp. 221--252, 2006.

\bibitem{yang2018mean}
\BIBentryALTinterwordspacing
Y.~Yang, R.~Luo, M.~Li, M.~Zhou, W.~Zhang, and J.~Wang, ``Mean field
  multi-agent reinforcement learning,'' \emph{arXiv:1802.05438}, 2018.
  [Online]. Available: \url{https://arxiv.org/abs/1802.05438}
\BIBentrySTDinterwordspacing

\bibitem{guo2019learning}
X.~Guo, A.~Hu, R.~Xu, and J.~Zhang, ``Learning mean-field games,'' in
  \emph{Proc. NeurIPS}, 2019, pp. 4966--4976.

\bibitem{subramanian2019reinforcement}
J.~Subramanian and A.~Mahajan, ``Reinforcement learning in stationary
  mean-field games,'' in \emph{Proc. AAMAS}, 2019, pp. 251--259.

\bibitem{perolat2021scaling}
J.~P{\'e}rolat, S.~Perrin, R.~Elie, M.~Lauri{\`e}re, G.~Piliouras, M.~Geist,
  K.~Tuyls, and O.~Pietquin, ``Scaling mean field games by online mirror
  descent,'' in \emph{Proc. AAMAS}, vol.~21, 2022, pp. 1028--1037.

\bibitem{guo2022mf}
\BIBentryALTinterwordspacing
X.~Guo, A.~Hu, and J.~Zhang, ``{MF-OMO}: An optimization formulation of
  mean-field games,'' \emph{arXiv:2206.09608}, 2022. [Online]. Available:
  \url{https://arxiv.org/abs/2206.09608}
\BIBentrySTDinterwordspacing

\bibitem{anahtarci2020q}
B.~Anahtarci, C.~D. Kariksiz, and N.~Saldi, ``{Q-learning} in regularized
  mean-field games,'' \emph{Dyn. Games and Appl.}, pp. 1--29, 2022.

\bibitem{cui2021approximately}
K.~Cui and H.~Koeppl, ``Approximately solving mean field games via
  entropy-regularized deep reinforcement learning,'' in \emph{Proc. AISTATS},
  2021, pp. 1909--1917.

\bibitem{campi2022correlated}
L.~Campi and M.~Fischer, ``Correlated equilibria and mean field games: a simple
  model,'' \emph{Math. Oper. Res.}, 2022.

\bibitem{muller2022learning}
\BIBentryALTinterwordspacing
P.~Muller, R.~Elie, M.~Rowland, M.~Lauriere, J.~Perolat, S.~Perrin, M.~Geist,
  G.~Piliouras, O.~Pietquin, and K.~Tuyls, ``Learning correlated equilibria in
  mean-field games,'' \emph{arXiv:2208.10138}, 2022. [Online]. Available:
  \url{https://arxiv.org/abs/2208.10138}
\BIBentrySTDinterwordspacing

\bibitem{gao2022energy}
H.~Gao, W.~Lee, Y.~Kang, W.~Li, Z.~Han, S.~Osher, and H.~V. Poor,
  ``Energy-efficient velocity control for massive numbers of {UAV}s: A mean
  field game approach,'' \emph{IEEE Trans. Veh. Technol.}, vol.~71, no.~6, pp.
  6266--6278, 2022.

\bibitem{shiri2019massive}
H.~Shiri, J.~Park, and M.~Bennis, ``Massive autonomous {UAV} path planning: A
  neural network based mean-field game theoretic approach,'' in \emph{Proc.
  IEEE GLOBECOM}, 2019, pp. 1--6.

\bibitem{wang2022mean}
G.~Wang, W.~Yao, X.~Zhang, and Z.~Li, ``A mean-field game control for
  large-scale swarm formation flight in dense environments,'' \emph{Sensors},
  vol.~22, no.~14, p. 5437, 2022.

\bibitem{vsovsic2018reinforcement}
A.~{\v{S}}o{\v{s}}i{\'c}, A.~M. Zoubir, and H.~Koeppl, ``Reinforcement learning
  in a continuum of agents,'' \emph{Swarm Intell.}, vol.~12, no.~1, pp. 23--51,
  2018.

\bibitem{carmona2019model}
\BIBentryALTinterwordspacing
R.~Carmona, M.~Lauri{\`e}re, and Z.~Tan, ``Model-free mean-field reinforcement
  learning: mean-field mdp and mean-field q-learning,''
  \emph{arXiv:1910.12802}, 2019. [Online]. Available:
  \url{https://arxiv.org/abs/1910.12802}
\BIBentrySTDinterwordspacing

\bibitem{gu2021mean}
H.~Gu, X.~Guo, X.~Wei, and R.~Xu, ``Mean-field controls with {Q-learning} for
  cooperative {MARL}: convergence and complexity analysis,'' \emph{SIAM J.
  Math. Data Sci.}, vol.~3, no.~4, pp. 1168--1196, 2021.

\bibitem{mondal2021approximation}
W.~U. Mondal, M.~Agarwal, V.~Aggarwal, and S.~V. Ukkusuri, ``On the
  approximation of cooperative heterogeneous multi-agent reinforcement learning
  ({MARL}) using mean field control ({MFC}),'' \emph{J. Mach. Learn. Res.},
  vol.~23, no. 129, pp. 1--46, 2022.

\bibitem{huttenrauch2019deep}
M.~H{\"u}ttenrauch, S.~Adrian, G.~Neumann \emph{et~al.}, ``Deep reinforcement
  learning for swarm systems,'' \emph{J. Mach. Learn. Res.}, vol.~20, no.~54,
  pp. 1--31, 2019.

\bibitem{fiorini1998motion}
P.~Fiorini and Z.~Shiller, ``Motion planning in dynamic environments using
  velocity obstacles,'' \emph{Int. J. Robot. Res.}, vol.~17, no.~7, pp.
  760--772, 1998.

\bibitem{hamer2018fast}
M.~Hamer, L.~Widmer, and R.~D’andrea, ``Fast generation of collision-free
  trajectories for robot swarms using gpu acceleration,'' \emph{IEEE Access},
  vol.~7, pp. 6679--6690, 2018.

\bibitem{everett2021collision}
M.~Everett, Y.~F. Chen, and J.~P. How, ``Collision avoidance in pedestrian-rich
  environments with deep reinforcement learning,'' \emph{IEEE Access}, vol.~9,
  pp. 10\,357--10\,377, 2021.

\bibitem{ourari2022nearest}
R.~Ourari, K.~Cui, A.~Elshamanhory, and H.~Koeppl, ``Nearest-neighbor-based
  collision avoidance for quadrotors via reinforcement learning,'' in
  \emph{Proc. IEEE ICRA}, 2022, pp. 293--300.

\bibitem{arulkumaran2017deep}
K.~Arulkumaran, M.~P. Deisenroth, M.~Brundage, and A.~A. Bharath, ``Deep
  reinforcement learning: A brief survey,'' \emph{IEEE Signal Process. Mag.},
  vol.~34, no.~6, pp. 26--38, 2017.

\bibitem{pasztor2021efficient}
\BIBentryALTinterwordspacing
B.~Pasztor, I.~Bogunovic, and A.~Krause, ``Efficient model-based multi-agent
  mean-field reinforcement learning,'' \emph{arXiv:2107.04050}, 2021. [Online].
  Available: \url{https://arxiv.org/abs/2107.04050}
\BIBentrySTDinterwordspacing

\bibitem{mondal2022on}
\BIBentryALTinterwordspacing
W.~U. Mondal, V.~Aggarwal, and S.~Ukkusuri, ``On the near-optimality of local
  policies in large cooperative multi-agent reinforcement learning,''
  \emph{Trans. Mach. Learn. Res.}, 2022. [Online]. Available:
  \url{https://openreview.net/forum?id=t5HkgbxZp1}
\BIBentrySTDinterwordspacing

\bibitem{parthasarathy2005probability}
K.~R. Parthasarathy, \emph{Probability measures on metric spaces}.\hskip 1em
  plus 0.5em minus 0.4em\relax American Mathematical Soc., 2005, vol. 352.

\bibitem{devore1993constructive}
R.~A. DeVore and G.~G. Lorentz, \emph{Constructive approximation}.\hskip 1em
  plus 0.5em minus 0.4em\relax Springer Science \& Business Media, 1993, vol.
  303.

\bibitem{liang2018rllib}
E.~Liang, R.~Liaw, R.~Nishihara, P.~Moritz, R.~Fox, K.~Goldberg, J.~Gonzalez,
  M.~Jordan, and I.~Stoica, ``{RLlib}: Abstractions for distributed
  reinforcement learning,'' in \emph{Proc. ICML}, 2018, pp. 3053--3062.

\bibitem{schulman2017proximal}
\BIBentryALTinterwordspacing
J.~Schulman, F.~Wolski, P.~Dhariwal, A.~Radford, and O.~Klimov, ``Proximal
  policy optimization algorithms,'' \emph{arXiv:1707.06347}, 2017. [Online].
  Available: \url{https://arxiv.org/abs/1707.06347}
\BIBentrySTDinterwordspacing

\bibitem{khatib1985real}
O.~Khatib, ``Real-time obstacle avoidance for manipulators and mobile robots,''
  in \emph{Proc. IEEE ICRA}, vol.~2, 1985, pp. 500--505.

\bibitem{villani2009optimal}
C.~Villani, \emph{Optimal transport: old and new}.\hskip 1em plus 0.5em minus
  0.4em\relax Springer, 2009, vol. 338.

\bibitem{schulman2015trust}
J.~Schulman, S.~Levine, P.~Abbeel, M.~Jordan, and P.~Moritz, ``Trust region
  policy optimization,'' in \emph{Proc. ICML}.\hskip 1em plus 0.5em minus
  0.4em\relax PMLR, 2015, pp. 1889--1897.

\bibitem{gupta2017cooperative}
J.~K. Gupta, M.~Egorov, and M.~Kochenderfer, ``Cooperative multi-agent control
  using deep reinforcement learning,'' in \emph{Proc. AAMAS}, 2017, pp. 66--83.

\bibitem{tan1993multi}
M.~Tan, ``Multi-agent reinforcement learning: Independent vs. cooperative
  agents,'' in \emph{Proc. ICML}, 1993, pp. 330--337.

\bibitem{de2020independent}
\BIBentryALTinterwordspacing
C.~S. de~Witt, T.~Gupta, D.~Makoviichuk, V.~Makoviychuk, P.~H. Torr, M.~Sun,
  and S.~Whiteson, ``Is independent learning all you need in the {Starcraft}
  multi-agent challenge?'' \emph{arXiv:2011.09533}, 2020. [Online]. Available:
  \url{https://arxiv.org/abs/2011.09533}
\BIBentrySTDinterwordspacing

\bibitem{yu2021surprising}
C.~Yu, A.~Velu, E.~Vinitsky, Y.~Wang, A.~Bayen, and Y.~Wu, ``The surprising
  effectiveness of {PPO} in cooperative, multi-agent games,'' \emph{Proc.
  NeurIPS Datasets and Benchmarks}, 2022.

\bibitem{papoudakis2021benchmarking}
G.~Papoudakis, F.~Christianos, L.~Sch{\"a}fer, and S.~V. Albrecht,
  ``Benchmarking multi-agent deep reinforcement learning algorithms in
  cooperative tasks,'' in \emph{Proc. NeurIPS Datasets and Benchmarks}, 2021.

\bibitem{fu2022revisiting}
W.~Fu, C.~Yu, Z.~Xu, J.~Yang, and Y.~Wu, ``Revisiting some common practices in
  cooperative multi-agent reinforcement learning,'' in \emph{Proc. ICML}, 2022,
  pp. 6863--6877.

\bibitem{giernacki2017crazyflie}
W.~Giernacki, M.~Skwierczy{\'n}ski, W.~Witwicki, P.~Wro{\'n}ski, and
  P.~Kozierski, ``Crazyflie 2.0 quadrotor as a platform for research and
  education in robotics and control engineering,'' in \emph{Proc. IEEE MMAR
  Conf.}, 2017, pp. 37--42.

\bibitem{greiff2019performance}
M.~Greiff, A.~Robertsson, and K.~Berntorp, ``Performance bounds in positioning
  with the vive lighthouse system,'' in \emph{Proc. IEEE FUSION}, 2019, pp.
  1--8.

\bibitem{caines2019graphon}
P.~E. Caines and M.~Huang, ``Graphon mean field games and the {GMFG} equations:
  $\varepsilon$-{Nash} equilibria,'' in \emph{Proc. IEEE CDC}, 2019, pp.
  286--292.

\bibitem{cui2022learning}
K.~Cui and H.~Koeppl, ``Learning graphon mean field games and approximate
  {Nash} equilibria,'' in \emph{Proc. ICLR}, 2022, pp. 1--31.

\bibitem{duan2022prevalence}
C.~Duan, T.~Nishikawa, and A.~E. Motter, ``Prevalence and scalable control of
  localized networks,'' \emph{PNAS}, vol. 119, no.~32, p. e2122566119, 2022.

\bibitem{perrin2021generalization}
S.~Perrin, M.~Lauri{\`e}re, J.~P{\'e}rolat, R.~{\'E}lie, M.~Geist, and
  O.~Pietquin, ``Generalization in mean field games by learning master
  policies,'' in \emph{Proc. AAAI}, vol.~36, no.~9, 2022, pp. 9413--9421.

\end{thebibliography}

\end{document}